\documentclass[letterpaper, 10 pt, conference]{ieeeconf}  

\IEEEoverridecommandlockouts                             

\usepackage{graphicx}
\usepackage{booktabs} 
\usepackage[utf8]{inputenc}

\usepackage{amsmath, amssymb, amsfonts, amsthm}

\usepackage{color}
\usepackage{bm}
\usepackage{hhline}
\usepackage{mathabx}
\usepackage{multirow}
\usepackage{setspace}
\usepackage{algorithm}
\usepackage{setspace}
\usepackage[noend]{algpseudocode}
\usepackage{algpseudocode}

\newtheorem{theorem}{Theorem}

\newtheorem{remark}[theorem]{Remark}

\usepackage[noadjust]{cite}

\title{\LARGE \bf
Learning and Concentration for High Dimensional Linear Gaussians: an Invariant Subspace Approach}

\author{Muhammad Abdullah Naeem 

\thanks{Author would like to acknowledge, helpful insights and feedback from Dr. Yuansi Chen}
\thanks{Muhammad Abdullah Naeem is with the Department of Electrical and Computer Engineering, 
        Duke University, Durham, NC 27708, USA, Email: 
        {\tt\small muhammad.abdullah.naeem@duke.edu}}%
}

\begin{document}

\maketitle
\thispagestyle{empty}
\pagestyle{empty}


\begin{abstract}
In this work, we study non-asymptotic bounds on correlation between two time realizations of stable linear systems with isotropic Gaussian noise. Consequently, via sampling from a sub-trajectory and using \emph{Talagrands'} inequality, we show that empirical averages of reward concentrate around steady state (dynamical system mixes to when closed loop system is stable under linear feedback policy ) reward , with high-probability. As opposed to common belief of larger the spectral radius stronger the correlation between samples,  \emph{large discrepancy between algebraic and geometric multiplicity of system eigenvalues leads to large invariant subspaces related to system-transition matrix}; once the system enters the large invariant subspace it will travel away from origin for a while before coming close to a unit ball centered at origin where an isotropic Gaussian noise can with high probability allow it to escape the current invariant subspace it resides in, leading to \emph{bottlenecks} between different invariant subspaces that span $\mathbb{R}^{n}$, to be precise  : system initiated in a large invariant subspace will be stuck there for a long-time: log-linear in dimension of the invariant subspace and inversely to log of inverse of magnitude of the eigenvalue. In the problem of Ordinary Least Squares estimate of system transition matrix via a single trajectory, this phenomenon is even more evident if spectrum of transition matrix associated to large invariant subspace is explosive and small invariant subspaces correspond to stable eigenvalues. Our analysis provide first interpretable and geometric explanation into intricacies of learning and concentration for random dynamical systems on continuous, high dimensional state space; exposing us to surprises in high dimensions and suggesting, whether it is a problem of system identification or policy evaluation, practitioner should avoid naive isotropic Gaussian excitations. Either pre-processing or a different choice of excitations should be chosen in  accordance with the worst case possible size of invariant subspaces they can encounter, ensuring `thorough' exploration of state space in minimum number of time steps.  

\end{abstract}

\section{Introduction}
\label{submission}
Over the last decade, we have seen a tremendous surge in sample complexity analysis for learning in control tasks. Whether it is the problem of learning value function corresponding to a control policy (see e.g. \cite{tu2018least}) or system identification as in (\cite{tsiamis2021linear}), analysis heavily relies on tedious probabilistic and analytic methods offering very less interpretation or geometric insights. As a result, uncertainty looms over our current understanding for learning of dynamical systems via single trajectory, and as we will discuss shortly afterwards; even a thorough understanding of stable Linear Gaussians(LGs) is absent. In this paper we conclude for good, sampling complexity, role of spectral radius and address more recent speculations about large noise being beneficial \cite{simchowitz2018learning}, \cite{oymak2021revisiting} and \cite{tsiamis2022online},\cite{tsiamis2022statistical}. In fact, it was recently pointed out by \cite{naeem2022concentration}, that decay of correlation between two distant samples of trajectory of a dynamical system is associated to spectral gaps(Functional analytic phenomenon), not the size of spectral radius of finite dimensional matrix. We provide a finite-dimensional interpretation of their result and throughout this paper we will assume working with \emph{high dimensional} underlying state space.

Recognizing these ambiguities, we study two simple problems in this paper. Assume that an unknown LG system is simulated under some stable policy $\pi$ and is assumed to have reached stationarity. In an ideal case, one would like to know expected reward w.r.t stationary distribution, but happens to only have access to time averages of reward. How good of an approximation are temporally-dependent time averages of LG for expected reward compared to i.i.d draws of reward from stationary distribution. It turns out that answer to this simple problem, explains all the queries in introduction. Secondly, we study the problem of Ordinary Least Squares (OLS) estimates for system identification via single trajectory of dynamical system. 
    
To the best of authors' knowledge we give first analysis of concentration and system identification by direct sum decomposition of original state transition matrix onto its' invariant subspaces. Along with Talagrand's concentration inequality and Gaussian projections on subspaces we are able to conclude, as opposed to standard beliefs, main issue in system identification and concentration is invariant subspaces of large dimensions with large magnitude of associated eigenvalue. If the size of an invariant subspace is large  and excitations are isotropic Gaussians', with overwhelming  probability majority of the excitation signal will lie in the large invariant subspace(Gaussian concentration of measure phenomenon). Adding to the complications, if the spectrum associated to the large invariant subspace is unstable or at a periphery between being stable or unstable,  action of projected state-transition operator on initial excitation will move the new realization away from origin while keeping it inside the large invariant subspace. As a result signal never explores smaller invariant subspaces leading to inaccurate learned behavior. Large invariant subspaces are a consequence of discrepancy between algebraic and geometric multiplicity of eigenvalues associated to state transition matrix. In contrast to our approach, most of the system identification work focuses on  crafting bounds on spectrum as a function of length of simulated trajectory, or initializing the system at origin and proving their results on low dimensional simulations. As we will show in simulation results on OLS in high dimensions, by projecting random initial excitations into large invariant subspaces, leads to incorrect estimates.

This phenomenon is also apparent when we study the problem of concentration of time averages around their spatial average(defined by the distribution LG mixes to, which only happens when spectral radius of system matrix is strictly less than 1).
Even after the associated Markov chain has  mixed to its' stationary distribution, only time-averages of a carefully chosen sub-trajectory concentrates around spatial average , because once the trajectory enters a large dimensional invariant subspace it will keep moving away from the origin (while being inside the aforementioned invariant subspace) till its' \emph{first contractive hitting time} when it is already en route origin and an isotropic Gaussian excitation will let it hop out to a different invariant subspace and the trajectory seems to \emph{regenerate}.
Although, sampling from a sub-trajectory of correlated samples give similar concentration results as i.i.d samples from stationary distribution, but the sub-trajectory is generated by selecting consecutive realizations, with a gap of :  maximum over all linearly independent invariant subspaces of `log-linear in dimension of the subspace and inversely to log of inverse of magnitude of the associated eigenvalue '   

The paper is organized as follows. In Section \ref{sec:Not and Prelim}, we introduce notation and preliminaries. Section \ref{sec:tensorization} develops two tensorizartion procedures for Talagrands' inequality for general dynamical system which will be at heart of understanding concentration of stable linear random dynamical systems and inconsistency of Ordinary Least Squares for explosive systems. In Section \ref{sec:stable-subjtrajectory}, we lay down some facts about stable LGs and propose sampling from sub-trajectory by leveraging upon Gelfands' formula. Section \ref{sec:invsub} gives a concise introduction to invariant subspaces associated to a state-transition matrix and non-asymptotic bounds for concentration via sampling from sub-trajectory are provided. We begin subsection A of Section \ref{sec:sys-id} with concise introduction to OLS problem via single trajectory and conclude with inconsistency of OLS (for explosive systems) under isotropic Gaussian exciations via a trivial application of tensorized Talagrands' inequality. In subsection B we present an isoperimetric 
approach to give an intuition of what may go wrong in high dimensions which leads to incorrect OLS estimates. Simulation results are presented in the Section \ref{sec:Simulation}, where we show as opposed to existing beliefs on consistency of OLS for regular systems, OLS in high dimensional regular system fails when large block of an invariants subspace corrresponds to an explosive eigenvalue. We conclude with a summary and direction on future work in \ref{sec:conclusion}.

\vspace{-2pt}
\section{Notation and Preliminaries}
\label{sec:Not and Prelim}
\vspace{-2pt}
\subsubsection{Notation}  We use ${I}_{n}\in\mathbb{R}^{n\times n}$ to denote the $n$ dimensional identity matrix. 
For random variables $x$ and $y$, $Cov(x,y)$ denote the covariance. 
$ \mathcal{B}_{\alpha}^{n}:=\{x \in \mathbb{R}^{n}:  \|x\|:= \|x\|_2
\leq \alpha \}$ is the 
$\alpha$-ball in $\mathbb{R}^n$. Similarly, $\mathcal{S}^{p-1}:=\{x \in \mathbb{R}^{p}:\|x\|_2=1\}$, is the unit sphere in $\mathbb{R}^{p}$.    $\chi_{\{\}}()$ is the indicator function, whereas $\rho(A)$,  $\|A\|_2$, $\|A\|_{F}$, $det(A)$, $tr(A)$ and $\sigma(A)$   represent the spectral radius,  matrix 2-norm , Frobenius norm, determinant, trace and set of eigenvalues(spectrum) of $A$ respectively. 
For a positive definite matrix $A$, largest and smallest eigenvaues are denoted by $\lambda_{max}(A)$ and $\lambda_{min}(A)$, respectively. Associated with every rectangular matrix $\mathbb{X} \in \mathbb{R}^{n \times N}$ are its' singular values $\sigma_{1}(\mathbb{X}) \geq \sigma_{2}(\mathbb{X}), \ldots, \sigma_{n}(\mathbb{X}) \geq 0$, where without loss of generality we assume that $N>n$. Of utmost importance is largest singular value,
$\sigma_{1}(\mathbb{X}):= \sup_{a \in \mathcal{S}^{N-1}} \|\mathbb{X}a\| $ and the least singular value $\sigma_{n}(\mathbb{X}):= \inf_{a \in \mathcal{S}^{N-1}} \|\mathbb{X}a\|$. Condition number of a matrix $\mathbb{X}$ is the ratio of the largest and least singular value, denoted by $\kappa(\mathbb{X})= \frac{\sigma_{1}(\mathbb{X})}{\sigma_{n}(\mathbb{X})}$. If the span of image space of $\mathbb{X}$ is $\mathbb{R}^{n}$, more compactly written as $Im(\mathbb{X})=\mathbb{R}^{n}$, least singular value equals the inverse of the norm of inverse of matrix $\mathbb{X}$ i.e., $\sigma_{n}(\mathbb{X})= \frac{1}{\|\mathbb{X}^{-1}\|}$. A function $g: \mathbb{R}^{n} \rightarrow \mathbb{R}^{p}$ is Lipschitz with constant $L$ if for every $x,y \in \mathbb{R}^{n} $, $\|g(x)-g(y)\| \leq L\|x-y\|$.  

A sequence ${\{a(N)\}_{N \in \mathbb{N}} \in  \mathcal{O}(N)}$, if it increases at most linearly in $N$ (this is not limited to asymptotic results). $\mathcal{O}(1)$ will be used to denote quantities independent of the size of the underlying state space or number of the iterations.
Space of probability measure on  $\mathcal{X}$(continuous space) is denoted by  $\mathcal{P(\mathcal{X})}$ and space of its Borel subsets is represented by $\mathbb{B}\big(\mathcal{P(\mathcal{X})}\big)$. For a function $r$ and $\mu \in \mathcal{P(X)}$, we use $<r>_{\mu}$ to denote expectation of $r$ w.r.t $\mu$. Finally, for a set $\mathcal{K}\subseteq\{1,...,M\}$, its complement is $\mathcal{K}^\complement:=\{1,...,M\}\setminus\mathcal{K}$. 

On a metric space $(\mathcal{X},d)$, for $\mu, \nu \in \mathcal{P(\mathcal{X})}$, we define Wasserstein metric of order $p \in [1, \infty)$~as
\begin{equation}
\label{eq:WM}
    \mathcal{W}_{d}^{p} (\nu,\mu)= \bigg(\inf_{(X,Y) \in \Gamma(\nu,\mu)} \mathbb{E}~d^{p}(X,Y)\bigg)^{\frac{1}{p}};
\end{equation}
here, $\Gamma(\nu,\mu) \in P(\mathcal{X}^{2})$, and $(X,Y) \in \Gamma(\nu,\mu)$ implies that random variables $(X,Y)$ follow some probability distributions on $P(\mathcal{X}^{2})$ with marginals $\nu$ and $\mu$. Another way of comparing two probability distributions on $\mathcal{X}$ is via relative entropy, which is defined as
\begin{equation}
\label{eq:ent} 
    Ent(v||u)=\left\{ \begin{array}{lr}
    \int \log\bigg(\frac{d\nu}{d\mu}\bigg) d\nu, & \text{if}~ \nu << \mu,
         \\ +\infty, & \text{otherwise}. 
         \end{array}\right.
\end{equation}

Before we introduce the mathematical framework to derive concentration for dependent random variables, we introduce the following results utilized later in this work. 

\paragraph{Talagrands' inequality or Transport-Entropy Inequality} Consider metric space $(\mathcal{X},d)$ and reference  probability measure $\mu \in P(\mathcal{X})$. Then we say that $\mu$ satisfies  $\mathcal{T}_{1}^d (C)$  or to be concise $\mu \in  \mathcal{T}_{1}^d (C)$ for some $C>0$ if for 
all $\nu \in P(\mathcal{X})$ it holds~that 
\begin{equation}
    \label{eq:t1}
    \mathcal{W}_d (\mu, \nu) \leq \sqrt{2 C Ent(\nu||\mu)}.
\end{equation}

\begin{theorem}[\cite{bobkov1999exponential}]
\label{lm:b-g}
$\mu$ satisfies $\mathcal{T}_{1}^d (C)$ if and only if for 
any Lipschitz function $f$ with $<f>_{\mu}:= \mathbb{E}_{\mu} f$, it holds that
\begin{flalign}
    \label{eq:bg} & \int e^{\lambda(f- <f>_{\mu})} d\mu \leq \exp(\frac{\lambda^2}{2}C \|f\|_{L(d)} ^2), \hspace{15pt}  \\ & \nonumber \text{where~~~~}  \|f\|_{L(d)}:= \sup_{x \neq y} \frac{|f(x)-f(y)|}{d(x,y)}.
\end{flalign}
\end{theorem}

\begin{remark}
\label{rm: iid}
(\ref{eq:bg}) along with the Markov inequality implies that if we sample $x$ from $\mu \in \mathcal{T}_{1}^d(C)$, then
\begin{equation}
    \label{eq:iid} \mathbb{P} \Bigg[ \bigg| r(x) -<r>_{\mu}\bigg| > \epsilon \Bigg] \leq 2\exp\bigg(-\frac{ \epsilon ^2}{2C \|r\|_{L(d)}^2}\bigg).
\end{equation}
\end{remark}
\section{Extending Concentration to Dependent Random Variables via Tensorization}
\label{sec:tensorization}
Under the action of some state dependent policy $\pi$, we consider a closed-loop random dynamical system of the form
\begin{equation}
\label{eq:crds}
    x_{k+1}=F\big(x_k, \pi(x_k), \epsilon_k\big), \hspace{10 pt} \text{with~}\epsilon_k  \hspace{10 pt} i.i.d, \footnote{For the sake of brevity, from now on we will exclude the reference to $\pi$ in the state update equations as a state-dependent policy implies there exists some function $G$ such that $F\big(x_k, \pi(x_k), \epsilon_k\big)=G(x_k, \epsilon_k)$.}
\end{equation}
where $x_{k} \in \mathbb{R}^n$ for all $k \in \mathbb{N}$ and $F: \mathbb{R}^n \times \mathbb{R}^n \times \mathbb{R}^n \longrightarrow \mathbb{R}^n $.
For the problem of concentration of ergodic averages, we will assume that the transition kernel converges to some stationary distribution $\mu_{\infty}$ under Wasserstein metric $W_{d}$ equipped with some distance function $d$. This random dynamical system can be viewed as a Markov chain  $ x^N:=(x_i)_{i=1}^{N}$ with distribution $\mu^{N} \in P$ $(\mathcal{X}^N)$ and $P^{m}(x,\mathcal{B}):= \mathbb{P}(x_m \in \mathcal{B}| x_0=x)$, for all Borel subsets $\mathcal{B}$ of $\mathcal{X}$. We can extend the metric $d$ to $\mathcal{X}^N$ as
\begin{equation}
    \label{eq:tensor}
    d_{(N)}(x^N,y^N)= \sum_{i=1}^{N} d(x_i,y_i).
\end{equation}
If $\mu^{N} \in \mathcal{T}_{1}^{d_{(N)}} \big( \mathcal{O} (N) \big)$ and $r$ is one Lipschitz, i.e., $\|r\| _{L(d)} \leq 1$, then $\Phi(x^{N}):= \frac{1}{N} \sum_{i=1}^{N} r(x_i) $ satisfies $\|\Phi\|_{L(d_{(N)})} \leq \frac{1}{N}$; plugging these results into~\eqref{eq:bg}, we obtain that
\begin{align}
     & \label{eq:MCC}
    \mu^{N} \Bigg[\bigg|\frac{1}{N} \sum_{i=1}^{N} r(x_i) -\mathbb{E}\Bigg(\frac{1}{N} \sum_{i=1}^{N} r(x_i) \Bigg) \bigg| > \epsilon \Bigg]  \\ & \nonumber  \leq 2\exp\bigg(-\frac{N \epsilon ^2}{2C}\bigg).
\end{align}

\subsection{Contractivity and Uniform Transport Constants }

As one would wonder from  \eqref{eq:tensor}, when does the T-E  for process level law of Markov chain, increases at worse linearly with dimension (in sample term)?
Sufficient conditions (see e.g., \cite{djellout2004transportation,bolley2005weighted})~are
 \begin{align}
    \label{eq:unfrmtran}
    (i)&\hspace{10pt} P(x,\cdot) \in T_1^{d}(C), \hspace{5pt} \text{for all}~{x \in \mathcal{X}}, \text{and some}~{C>0},  \\
    \label{eq:wascon}
    (ii)&\hspace{10pt}  \mathcal{W}_{d}(P(x,\cdot),P(y,\cdot))  \leq \hat{\lambda}d(x,y),\hspace{5pt} \text{for all }~{(x,y) \in \mathcal{X}^{2}} 
    \end{align}
 and some $\hat{\lambda} \in [0,1)$. 
  
 Property \eqref{eq:unfrmtran} is often referred to as existence of a uniform transportation constant and  \eqref{eq:wascon} represents contractivity of the Markov Chain in the Wasserstein metric / \emph{spectral gap in the Wasserstein sense}. 
 Now, 
 the following result holds.
 
 \begin{theorem}
\label{cl:tranN} If \eqref{eq:unfrmtran} and  \eqref{eq:wascon} hold,  process level distribution of samples from a Markov chain $(x_1,\ldots,x_N)$, which we will 
 denote as $Law(x_1,\ldots,x_N)$, 
 denoted by $\mathcal{\mu}^{N}$ satisfies  $T_{1}^{d_{(N)}} \bigg( \frac{CN}{(1-\hat{\lambda})^2}\bigg)$, for all $N \in \mathbb{N}$.
 \end{theorem}
 
\begin{proof} 
See Theorem  2.5 of \cite{djellout2004transportation} for a detailed proof.
\end{proof}
\paragraph{Decay of correlation.} By combining conditions from \eqref{eq:unfrmtran} and \eqref{eq:wascon}, with Taylor's expansion for small $\lambda$ (terms of order up to $\lambda^2$) appearing on both sides in Bobkov-Gotze dual form \eqref{eq:bg}, 
for all $x \in \mathcal{X}$ it holds that 
\vspace{-10pt}
\begin{align}
    \label{eq:deccor} |Cov_{P_x} [f(x_n),f(x_{n+k})]| \leq \frac{\hat{\lambda}^{k}}{1-\hat{\lambda}^2} C \|f\|_{L(d)} ^{2}. 
\end{align}

Another tensorization procedure that we will heavily rely on in Section \ref{sec:sys-id}, to better understand deviation inequalities for singular values of data matrix, is an an $\ell^{2}$ inspired metric on $\mathcal{X}^{N}$ as in \cite{malrieu2001logarithmic}, i.e., 
\vspace{-4pt}
\begin{equation}
    \label{eq:l2ten} 
    d_{(N)} ^2 (x^N,y^N):= \sqrt{\sum\nolimits_{i=1}^{N} d^2(x_i,y_i)}.
\end{equation}

\section{STABLE DYNAMICAL SYSTEMS AND ALMOST INDEPENDENT SUBTRAJECTORY }
\label{sec:stable-subjtrajectory}
\subsection{Independent sampling from invariant measure}
Markov chain under consideration is $n$ dimensional LG with isotropic noise:
\begin{equation}
\label{eq:LGS}
    x_{t+1}= Ax_t+ w_{t}, \hspace{10pt} \rho(A) <1 \hspace{10pt} \text{and i.i.d }~ w_{t} \thicksim \mathcal{N}(0,\mathcal{I}_n).
\end{equation}
It mixes to stationary distribution $\mu_{\infty} \thicksim \mathcal{N}(0, P_{\infty})$, where the controllability grammian $P_{\infty}$ is the unique positive definite solution of the following Lyapunov equation:
\begin{equation}
\label{eq:contgram}
    A^{T}P_{\infty}A-P_{\infty}+I_{n}=0.
\end{equation}

\begin{align}
     & \label{eq:inviid}
    \mathbb{P} \Big[\big|\frac{1}{N} \sum_{i=1}^{N} r(x_i) -\mu_{\infty}(r)  \big| > \epsilon \Big]    \leq 2\exp\bigg(-\frac{N \epsilon ^2}{2\lambda_{max}(P_{\infty})}\bigg).
\end{align}
\subsection{Sampling from a Sub-trajectory of LGs}
\label{sub:sec3-sample-sub-trajectory}
Stability in controls community for Linear systems correspond to  $\rho(A)<1$ (marginally stable corresponds $\rho(A) \geq 1$ and explosive system when $\rho(A)>1$) and it is an established result in real analysis (Gelfands formula) that for all $\rho \in (\rho(A),1)$ and $k \in \mathbb{N}$, there exists a finite positive constant $L_{\rho}$ such that $\|A^{k}\| \leq L_{\rho} \rho^{k}$. Naive Wasserstien contractivity condition \eqref{eq:wascon} does not hold as $L_{\rho}\rho$ can be very large.    In order to extend preceding result to spectral radius case, we define \emph{first contractive hitting time} as:
\begin{equation}
\label{eq:firsthit}
    \hat{k}:= \min \{k \in \mathbb{N} : \|A^k\|<1\},
\end{equation}
that is smallest natural number such that  $\hat{k}$-th step transition kernel $P(x_{\hat{k}}= \cdot|x_0)$ contractive in Wasserstein sense.

So, instead of considering the original trajectory 
$(x_0,x_1,x_2, \ldots)$, we will consider the sub-trajectory $(x_0, x_{\hat{k}}, x_{2\hat{k}}, \ldots)$ with modified LG dynamics:
\begin{align}
    \label{eq:subtrajLDS}
    x_{\hat{k}(i+1)}=A^{\hat{k}} x_{\hat{k}(i)}+s_{\hat{k}(i)} 
\end{align}
where $s_{\hat{k}(i)} \thicksim \mathcal{N}\big(0,\Sigma_{\hat{k}} \big)$, i.i.d with  $\Sigma_{\hat{k}}:=\sum_{l=0}^{\hat{k}-1} [A^l][A^l]^T$ $\forall i \in \mathbb{N}$. Notice that $\Sigma_{\hat{k}}$ is positive definite: Consequently,
\begin{align}
    \label{eq:subunfrmtran}
    (i)&\hspace{10pt} P^{\hat{k}}(x,\cdot) \in T_1^{d}\bigg(\|\Sigma_{\hat{k}}^{\frac{1}{2}}\|^2\bigg), \hspace{65pt} \text{for all}~{x \in \mathcal{X}}.  \\
    \label{eq:subwascon}
    (ii)&\hspace{10pt}  \mathcal{W}_{d}(P^{\hat{k}}(x,\cdot),P^{\hat{k}}(y,\cdot))  \leq \lambda_{\hat{k}}d(x,y), 
    \end{align}
for all  ${(x,y) \in \mathcal{X}^{2}}$  and ${\lambda_{\hat{k}}:= \|A^{\hat{k}}\| \in [0,1)}$. $\mu_{\hat{k}}^{N}:=Law(x_{\hat{k}(1)}, \ldots, x_{\hat{k}(N)} ) \in T_{1}^{d_{(N)}} \bigg( \frac{\|\Sigma_{\hat{k}}^{\frac{1}{2}}\|^2 N}{(1-\|A^{\hat{k}}\|)^2}\bigg)$ and from \ref{cl:tranN} if we start $x_{0} \thicksim \mu_{\infty}$, we have the following concentration:
\begin{align}
     & \label{eq:invcov}
    \mathbb{P} \Big[\big|\frac{1}{N} \sum_{i=1}^{N} r(x_{\hat{k}(i)}) -\mu_{\infty}(r)  \big| > \epsilon \Big]  \leq 2\exp\bigg(-\frac{N \epsilon ^2[1-\|A^{\hat{k}}\|]^2}{2\lambda_{\max}(P_{\infty})}\bigg) \\ & \label{eq:stationarysamp}= 2 \exp \bigg( -\frac{N \epsilon^2 [1-\|A^{\hat{k}}\|]^2 \lambda_{\min}(P_{\infty}^{-1})}{2}\bigg).
\end{align}
Therefore, compared to i.i.d samples from $\mu_{\infty}$, temporally dependent although identically distributed samples(each individually distributed as $\mu_{\infty}$)  spaced $\hat{k}$ times apart can give us sharp concentration but we need a trajectory of length $N\hat{k}$. Notice that as $\hat{k} \rightarrow \infty $ stationary chain concentration becomes i.i.d concentration. For a detail analysis of non-stationary case we refer to \cite{naeem2022transportation}. In order to bound first contractive hitting time we need to understand: 

\section{Structure of invariant sub spaces associated to eigenvalue problem of a non-symmetric operator}
\label{sec:invsub}
Position or magnitude of eigenvalues associated to a linear operator $A$ only provides partial information about its' properties (for the ease of exposition, throughout this paper we will assume that $A$ does not have any non-trivial null space). In fact knowing $A$ is equivalent to knowing its' invariant subspaces (see e.g., \cite{kato2013perturbation}).  Roughly speaking, algebraic multiplicity of eigenvalues follow from determinant of the matirx. 
\begin{equation}
    \label{eq:det} det(zI-A)= \prod_{i=1}^{K} (z-\lambda_{i})^{m_i},
\end{equation}
where $\lambda_{i}$ are distinct with multiplicity $m_{i}$. Complication happens when $dim[N(A-\lambda_{i}I)]<m_{i}$, which leads to invariants subspace (spanned by more that one linearly independent vector). Consequently, states space can be written as direct sum decomposition of $A-$ invariant subspaces.  
\begin{equation}
\label{eq:directsumA}
    \mathbb{R}^{n}= M_{\phi(1)} \oplus M_{\phi(2)} \oplus \ldots \oplus M_{\phi(L)}
\end{equation}
 and respective orthogonal projetcions $[E_{\phi(i)}]_{i=1} ^{L}$ such that identity matrix can be written as:
\begin{equation}
    \label{eq:addId} I_{n}=E_{\phi(1)} \oplus E_{\phi(2)} \oplus \ldots \oplus E_{\phi(L)}
\end{equation} 
where $\phi$ is a surjective map from $\{1,\ldots,L\}$ to $\sigma(A)$. $\phi$ is bijective iff eigenvectors span $\mathbb{R}^{n}$. \emph{In the case of gap between between algebraic and geometric multiplicity related to some element of $\sigma(A)$}.
Consider the invariant subspace $M_{\lambda}$, for some $\lambda \in \sigma(A)$, with algebraic multiplicity of $\lambda$ is $|B_{\lambda}|$ but only one linearly independent eigenvector $v_1$ such that $Av_1=\lambda v_1$. So we generate generalized eigenvector $v_2, v_3, \ldots, v_{|B_{\lambda}|}$ recursively as $(A-\lambda I)v_2=v_1$ and $(A-\lambda I)v_3=v_2$ and so on. We have the following $k$ -th step iteration:
\begin{align}
     & \nonumber A^{k}v_1 =\lambda^{k}v_1 \\ & \nonumber A^{k}v_2=\lambda^{k}v_{2} + \binom{k}{1}\lambda^{k-1}v_1 \\ & \nonumber A^{k}v_3=\lambda^{k}v_{3} + \binom{k}{1}\lambda^{k-1}v_2 +  \binom{k}{2} \lambda^{k-2}v_1 \\ & \nonumber \ldots= \ldots \\ & \nonumber
    A^{k}v_{|B_{\lambda}|}= \lambda^{k}v_{|B_{\lambda}|}+ \binom{k}{1}\lambda^{k-1}v_{|B_{\lambda}|-1}+\ldots \\ & \label{eq:recursiveJordan} + \binom{k}{|B_{\lambda}|-1}v_1
\end{align}

\begin{theorem} Although evident from the preceding  iterations, we can rigorously upper and lower bound norm of the $k-th$ iteration associated to action of matrix $A$ on invariant subspace $M_{\lambda}$, precisely given as:  
    \begin{align}
       |\lambda|^{k} \sum_{m=0}^{|B_{\lambda}|-1} \frac{1}{|\lambda|^{m}} \leq  \|A^{k}_{M_{\lambda}}\|_{2} \leq |\lambda|^{k} k^{|B_{\lambda}|}\sum_{m=0}^{|B_{\lambda}|-1} \frac{1}{|\lambda|^{m}}, 
    \end{align}
where $A^{k}_{M_{\lambda}}:=A^{k}E_{\lambda}$    
\end{theorem}
\begin{proof}
 The result follows by a simple variation of bounds provided in \cite{tropp2001elementary}.

\begin{align}
    \label{eq:normtpltz}  |\lambda|^{k} \sum_{m=0}^{|B_{\lambda}|-1} \frac{1}{|\lambda|^{m}} \leq \|A^{k}_{M_{\lambda}}\|_{2} & = \sum_{m=0}^{|B_{\lambda}|-1} \binom{k}{m} |\lambda|^{k-m} \\ & \nonumber \leq |\lambda|^{k} k^{|B_{\lambda}|} \sum_{m=0}^{|B_{\lambda}|-1} \frac{1}{|\lambda|^{m}} 
\end{align}
 In the first and second inequality we have used the fact  $1 \leq \binom{k}{m} \leq k^{|B_{\lambda}|}$ for $m \in [0, \ldots, B_{|\lambda|}-1 ]$

\end{proof}
\begin{remark}
    Moreover, if $|\lambda| \in (0,1)$ then:
    \begin{equation}
        \|A^{k}_{M_{\lambda}}\|_{2} \leq k^{|B_{\lambda}|}|B_{\lambda}|  |\lambda|^{k+1- |B_{\lambda}| } 
    \end{equation}
and
\begin{equation}
    \label{eq:ycln}
    k \geq \frac{\ln(|B_{\lambda}|)}{\ln(\frac{1}{|\lambda|})} + \frac{|B_{\lambda}| \ln (k)}{\ln(\frac{1}{|\lambda|})} + (|B_{\lambda}|-1)
\end{equation}
suffices for  $\|A^{k}_{M_{\lambda}}\|_{2} <1$.     
\end{remark}
Now we are in a position to give a conclusive analytic remark on the how to pick a sub-trajectory to get sharp concentration  for time averages around their spatial average as raised in subsection \ref{sub:sec3-sample-sub-trajectory}.
\begin{theorem}
    First contractive hitting time for  operator $A$ restricted to invariant subspace $M_{\phi(i)}$ is $\mathcal{O} \bigg( \frac{|B_{\phi(i)}| \ln |B_{\phi(i)}|  }{\ln(\frac{1}{|\phi(i)|})}\bigg)$. Therefore, contractive hitting time for linear operator $A$ is the worst contractive hitting time over all invariant subspaces.
\begin{equation}
    \hat{k}= \min \big[k \in \mathbb{N} : k \geq \max_{i \in 1, \ldots, L} \bigg(\frac{4|B_{\phi(i)}| \ln |B_{\phi(i)}| }{\ln \frac{1}{|\lambda_{\phi(i)}|}} \bigg) \big],
    \end{equation} 
which is also verified via simulations shown in Fig \ref{fig:2figsA}.    
\end{theorem}
\begin{figure}[!t]
\centering
\parbox{6.6cm}{
\includegraphics[width=6.6cm]{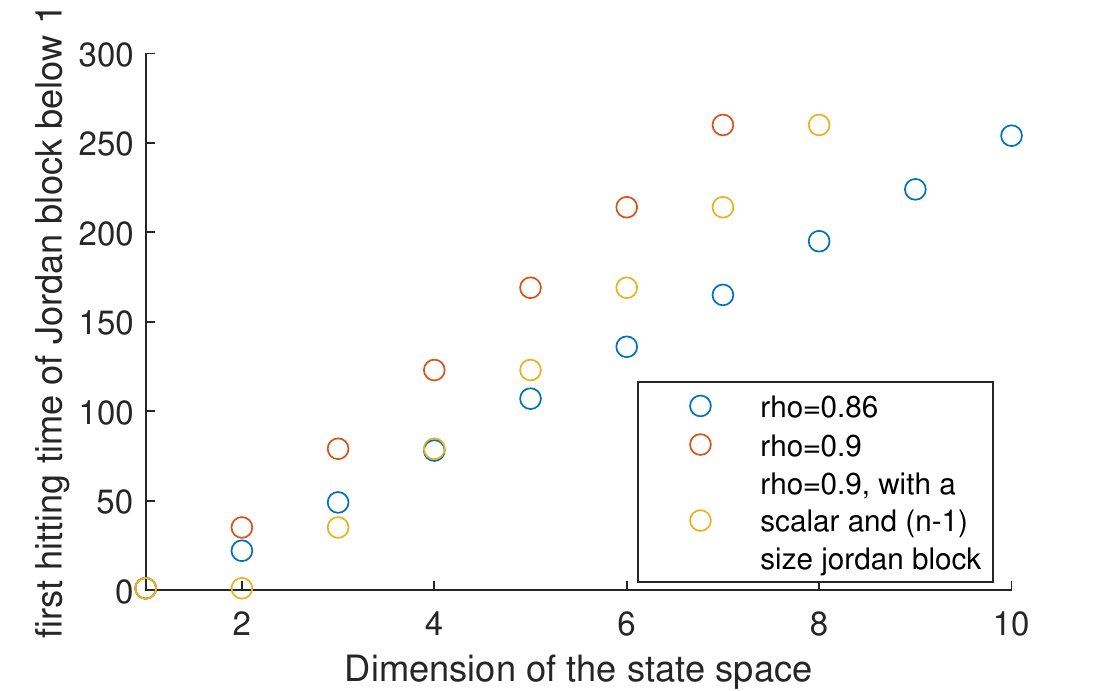}
\vspace{-18pt}
\caption{First contractive hitting time of $n \times n$ Jordan forms. case 1: single Jordan block with eigenvalue of 0.86. case 2: single Jordan block with eigenvalue of 0.9. case 3: $(n-1) \times  (n-1)$ Jordan block with eigenvalue 0.9 and a single block with eigenvalue 0.9. }
\label{fig:2figsA}}

\end{figure}

\begin{remark}
    Since a very lose upper-bound on first contractive hitting time is $\bigg(\frac{n \ln n }{\ln \frac{1}{\rho(A)}}\bigg)$, most of the literature on system identification for stable system uses this as an overhead for choosing block sizes (see e.g., \cite{simchowitz2018learning} and \cite{tsiamis2021linear}).   
\end{remark}

As $\|A^{\hat{k}}\|$ is contractive in Wasserstein sense, we have the following exponential convergence of the sub-trajectory \eqref{eq:subtrajLDS} to stationary distribution: 
\begin{theorem}
\label{thm:1contwas}
$\lambda_{\hat{k}}:=\|A^{\hat{k}}\|_{2}<1$ and for all $m \in \mathbb{N}$
\begin{align}
 \nonumber
    W_{d}(P_{x}^{\hat{k}(m)}, \mu_{\infty}) & \leq  \lambda_{\hat{k}} ^{m} W_{d}(P_{x}^{\hat{k}}, \mu_{\infty})\\ \label{eq:wasmix} & \leq \lambda_{\hat{k}} ^{m} \sqrt{\lambda_{\hat{k}}\|x\|+Tr\big([\sqrt{\Sigma_{\hat{k}}}-\sqrt{P_{\infty}}]^2\big)}.
\end{align}
\end{theorem} 
\begin{proof}
It suffices to show that for all $x, y$ in $\mathbb{R}^{n}$, $W_{d}(P_{x} ^{\hat{k}}, P_{y} ^{\hat{k}} ) \leq \lambda_{\hat{k}}\|x-y\|$, see e.g., \cite{hairer2011yet}.
Given any $x \in \mathbb{R}^{n}$, we can write it as a direct sum $x=\sum_{i=1}^{L} E_{\phi(i)}x$. Leveraging on orthogonality: $E_{\phi(i)}E_{\phi(j)}=0$ for $i \neq j$, we also have $\|x\|_{2}= \sum_{i=1}^{L}\|E_{\phi(i)}x \|_{2}$. Now, $\|A^{\hat{k}}x\|_{2}=\|A^{\hat{k}}\sum_{i=1}^{L} E_{\phi(i)}x\|_{2} \leq \sum_{i=1}^{L} \|A_{M_{\phi(i)}} ^{\hat{k}}\|_{2} \|E_{\phi(i)}x\|_{2} \leq \lambda_{\hat{k}} \|x\|_{2}$, where the last inequality follows from hypothesis and first contractive hitting time for individual block given in \eqref{eq:ycln}. Therefore  $W_{d}(P_{x} ^{\hat{k}}, P_{y} ^{\hat{k}} ) \leq \lambda_{\hat{k}}\|x-y\|$ for all $x, y$ in $\mathbb{R}^{n}$ and the result follows by realizing general expression for Wasserstein distance between two Gaussians in terms of their mean and covariance see e.g., \cite{givens1984class}. 
\end{proof}
 
\section{SYSTEM IDENTIFICATION VIA SINGLE TRAJECTORY}
\label{sec:sys-id}
\subsection{Ordinary Least Square}
In this section we analyse the problem of OLS estimation for system transition matrix $A$ from single observed (as in \cite{sarkar2019near}, \cite{simchowitz2018learning}, \cite{tsiamis2021linear}) trajectory of $(x_0,x_1, \ldots,x_{N})$ satisyfing:
\begin{equation}
    x_{t+1}=Ax_{t}+\eta_{t}, \hspace{10pt} \text{ where } \eta_{t} \thicksim \mathcal{N}(0,\mathcal{I}_n) .
\end{equation}  
OLS solution is:
\begin{equation}
    \label{eq:OLSsol} \hat{A}= \arg \min_{B \in \mathbb{R}^{n \times n}}  \sum_{t=0}^{N-1} \|x_{t+1}-Bx_{t}\|.
\end{equation}
Let $\mathbb{X}_{+}=[x_1, x_2,  \ldots, x_N]$ and $ \mathbb{X}_{-}=[x_0, x_1, \ldots, x_{(N-1)}]$, and noise covariates $E=[\eta_0, \eta_1, \ldots, \eta_{N-1}]$ then the closed form expression for  Least squares solution and error are:
\begin{align}
 & \label{eq:OLS}    \hat{A}= \mathbb{X}_{+}\mathbb{X}_{-} ^{T}(\mathbb{X}_{-}\mathbb{X}_{-}^{T})^{-1}  \\ & \label{eq:OLSerror} \|A-\hat{A}\| =\|E\mathbb{X}_{-}^{T} (\mathbb{X}_{-}\mathbb{X}_{-}^{T})^{-1}\|
\end{align}
Error can be upper bounded:
\begin{align}
    \nonumber \|A-\hat{A}\|  =\|E\mathbb{X}_{-}^{T} (\mathbb{X}_{-}\mathbb{X}_{-}^{T})^{-1}\| & \leq \|E\| \sigma_{1}(\mathbb{X}_{-}) \frac{1}{\sigma_{n}^{2}(\mathbb{X}_{-})}\\ & \label{eq:ubderrlst} =\frac{\sigma_{1}(E)\kappa(\mathbb{X}_{-})}{\sigma_{n}(\mathbb{X}_{-})}, 
\end{align}
where, recall $\kappa(\mathbb{X}_{-})$ is the condition number of $\mathbb{X}_{-}$. It is a well known result in Random Matrix Theory(see e.g., \cite{rudelson2009smallest}) if $n$ and $N$ are increased while maintaining their ratio $\frac{n}{N}= \gamma \in (0,1)$, then:
\begin{align}
    \sigma_{1}(E) \thicksim \sqrt{N}+\sqrt{n}, 
\end{align}
where $\thicksim$ here denotes typical behavior explained in discussion below Theorem \ref{thm:gausssubcon}. So we are left with task of bounding singular values of data matrix $\mathbb{X}_{-}$ (which contains dependent random variables, but here we will see Talagrands' inequality in all of its' glory )
\begin{theorem}
\label{thm:OLSvarhuge}
We have the following conentration bounds on all singular values $[\sigma_{k}(\mathbb{X}_{-})]_{k=1}^{n}$ of the data matrix 
\begin{enumerate}
    \item if $\|A\|_{2}<1 $: 
        \begin{align}
            & \nonumber \mathbb{P} \Big[\big|\sigma_{k}(\mathbb{X}_{-}) -\mathbb{E}\sigma_{k}(\mathbb{X}_{-}) \big| > \epsilon \Big] \\ & \nonumber  \leq 2 \exp \bigg(- \epsilon^{2}[1-\|A\|]^2\bigg)
        \end{align}
    \item if $\|A\|_{2}=1$ : 
        \begin{align}
             & \nonumber \mathbb{P} \Big[\big|\sigma_{k}(\mathbb{X}_{-}) -\mathbb{E}\sigma_{k}(\mathbb{X}_{-}) \big| > \epsilon \Big] \\  & \nonumber  \leq 2 \exp \bigg(-\frac{\epsilon^2 [e-1] }{N(N+1)}\bigg) 
    \end{align}
    \item $\|A\|_{2}>1$ :
        \begin{align}
             & \nonumber \mathbb{P} \Big[\big|\sigma_{k}(\mathbb{X}_{-}) -\mathbb{E}\sigma_{k}(\mathbb{X}_{-}) \big| > \epsilon \Big] \\  & \nonumber  \leq 2 \exp\bigg(-\frac{\epsilon^2 [N-1] }{\|A\|^{N}e(N+1)}\bigg)
        \end{align}
\end{enumerate}
\end{theorem}
\begin{proof}
 The idea of the proof follows from tensorization of Talagrands' inequality for dependent covariates. Notice that for any $a \in S^{N-1}$:
\begin{equation}
    (x_0, x_1, \ldots, x_{N-1}) \mapsto \mathbb{X}_{-}a \mapsto \|\mathbb{X}_{-}a\|_{2} 
\end{equation}
is a 1-Lipschitz map from $(\mathbb{R}^{n})^{N}$ to $\mathbb{R}$ under metric $d_{(N)} ^2$. Since $S^{N-1}$ is compact, $\inf$ and $\sup$ are attained, least and largest singular values are 1-Lipschitz (so are intermediate singular values via their min-max characterization/ Courant-Fischer theorem). Let $\mu^{N}:=Law(x_0, x_1, \ldots, x_{N-1} )$
\begin{itemize}
\item if $\|A\|_{2}<1$:
    \begin{equation}
        \mu^{N} \in T_{1}^{d_{(N)} ^2} \bigg( \frac{1}{[1-\|A\|_{2}]^2}\bigg) 
    \end{equation}
\item if $\|A\|_{2}=1$:
    \begin{equation}
        \mu^{N} \in T_{1}^{d_{(N)} ^2} \bigg( \frac{N(N+1)}{e-1}\bigg) 
    \end{equation}
\item if $\|A\|_{2}>1$:    
    \begin{equation}
        \mu^{N} \in T_{1}^{d_{(N)} ^2} \bigg( \frac{\|A\|^{N}e(N+1)}{N-1} \bigg)
    \end{equation}
\end{itemize}
see Proposition 4.1 in \cite{blower2005concentration} and the result follows.
\end{proof}
\begin{remark}
    Since the case $\|A\|_{2}>1$ includes explosive  systems and variance seems to deteriorate with number of samples; therefore, least singular value does not concentrate and OLS is inconsistent. Also notice that how tensorization of a sub-trajectory for stable systems give better concentration estimates as $\mu_{\hat{k}}^{N}:=Law(x_{\hat{k}(1)}, \ldots, x_{\hat{k}(N)} ) \in T_{1}^{d_{(N)} ^2} \bigg( \frac{\|\Sigma_{\hat{k}}^{\frac{1}{2}}\|^2 }{(1-\|A^{\hat{k}}\|)^2}\bigg)$ and consequently, dimension free concentration inequalities for its singular values. 
    \end{remark}
Now we provide an intuitive explanation of why OLS is inconsistent for explosive systems via isoperimetric reasonings.
\subsection{Projection of Isortopic Gaussian}
\begin{remark}   
This subsection is only for instructive purposes, employing notations like approximately, which we do not justify as the results here are isoperimetric in nature which we plan on considering for future work but even now intuitively explains what may go wrong in learning for high dimensional dynamical systems
\end{remark}
\begin{theorem}
\label{thm:gausssubcon}
    Let $\gamma_{n}$ be isotropic Gaussian in $\mathbb{R}^{n}$  and  $ S \subset \mathbb{R}^{n}$ be a $k-$ dimensional subspace. Given $x \in \mathbb{R}^{n}$, let $x_{S}$ denote the projection of $x$ onto $S$. Then for any $\delta \in (0,1)$
\begin{align}
& \nonumber \gamma_{n}\bigg(x \in \mathbb{R}^{n} : \frac{\|x_{S}\|}{\|x\|} \geq (1-\delta)^{-1}\sqrt{\frac{k}{n}}\bigg) \leq e^{-\frac{\delta^2 k}{4}}+ e^{-\frac{\delta^2 n}{4}} \\ & \label{eq:concsubspace}\gamma_{n}\bigg(x \in \mathbb{R}^{n} : \frac{\|x_{S}\|}{\|x\|} \leq (1-\delta)\sqrt{\frac{k}{n}}\bigg) \leq e^{-\frac{\delta^2 k}{4}}+ e^{-\frac{\delta^2 n}{4}},
\end{align}
see Lemma 3.2 in \cite{barvinok2005math}.
\end{theorem}
That is ratio of the norm of projection onto a $k-$ dimensional subspace is \emph{typically} ($\thicksim$) $\sqrt{\frac{k}{n}}$ and a remarkable advantage of this observation is that we can apply it on $A$- invariant sub-spaces to get an intuitive understanding of bottlenecks between invariant sub-spaces. Formally speaking if $E_{\lambda}$ is a projection onto $A-$invariant subspace $M_{\lambda}$ than:
\begin{equation}
    \label{eq:projcetiongaussian}
    \frac{\|E_{\lambda} \eta_0\|}{\|\eta_0\|} \thicksim \sqrt{\frac{|B_{\lambda}|}{n}}
\end{equation}
and as the size of underlying state space increases and so does the size of invariant subspace while maintaining $n-|B_{\lambda}|=\mathcal{O}(1)$ we can almost certainly conclude that almost all the excitation signal $\eta_0$ lies inside $M_{\lambda}$ (follows from  \eqref{eq:concsubspace})  
\begin{remark}
    Estimate in \eqref{eq:concsubspace} is at the heart of Johnson-Lindenstrauss (JL) lemma which is a very powerful tool for dimensionality reduction in high dimensional euclidean space to circumvent curse of dimensionality, see e.g., Lemma 2.1 of chapter 1 in \cite{ambrosiooptimal}
\end{remark}
Conditioned on $x_0=0$, we can express $N-th$ realization of the signal as:
\begin{equation}
\label{eq:op-noise}
    x_{N}= \sum_{t=1}^{N} A^{N-t} \eta_{t-1}.
\end{equation}

    For the explosive case i.e., $\lambda \in \sigma(A)$ such that $|\lambda|>1$, if  $\eta_{0} \thicksim \mathcal{N}(0,I_n)$, with overwhelming probability we have the following norm bound
\begin{equation}
    \label{eq:uplbtoep} |\lambda|^{k-|B_{\lambda}|+1} \bigg(\frac{|B_{\lambda}|^{\frac{3}{2}}}{n^{\frac{1}{2}}}\bigg) \lesssim  \frac{\|A^{k} E_{\lambda} \eta_0\|}{\|\eta_0\|} \lesssim k^{|B_{\lambda}|}|\lambda|^{k}\bigg(\frac{|B_{\lambda}|^{\frac{3}{2}}}{n^{\frac{1}{2}}}\bigg)
\end{equation}
Since $A^{k}E_{\lambda}=A^{k}_{M_{\lambda}}E_{\lambda}$, signal $A^{k} E_{\lambda} \eta_0$ lies entirely inside subspace $M_{\lambda}$ and \eqref{eq:uplbtoep} can be interpreted as: if the block size is huge and corresponding eigenvalue is explosive, predominantly the realizations of our dynamical system will lie inside $M_{\lambda}$ and consequently any realistic algorithm would fail to learn the behavior of the system in other parts of the state space.

Assume that $\mathbb{R}^{n}= M_{\lambda} \oplus M_{\lambda}^{\perp}$, with $\sigma(A_{M_{\lambda}^{\perp}})$ being stable with largest eigenvalue $\lambda^{\perp} \in (0,1)$. Since $n-|B_{\lambda}|= \mathcal{O}(1)$, implies that $|B_{\lambda}^{\perp}|= \mathcal{O}(1)$ it trivially follows that for even small values of $N$

\begin{align}
     \frac{\|A^{N} E_{\lambda}^{\perp} \eta_0\|}{\|A^{N} E_{\lambda} \eta_0 \|} \approx 0.
\end{align}

\section{Simulation Result on OLS}
\label{sec:Simulation}
\begin{figure}[!t]
\centering
\parbox{6.6cm}{
\includegraphics[width=7cm,height=6cm]{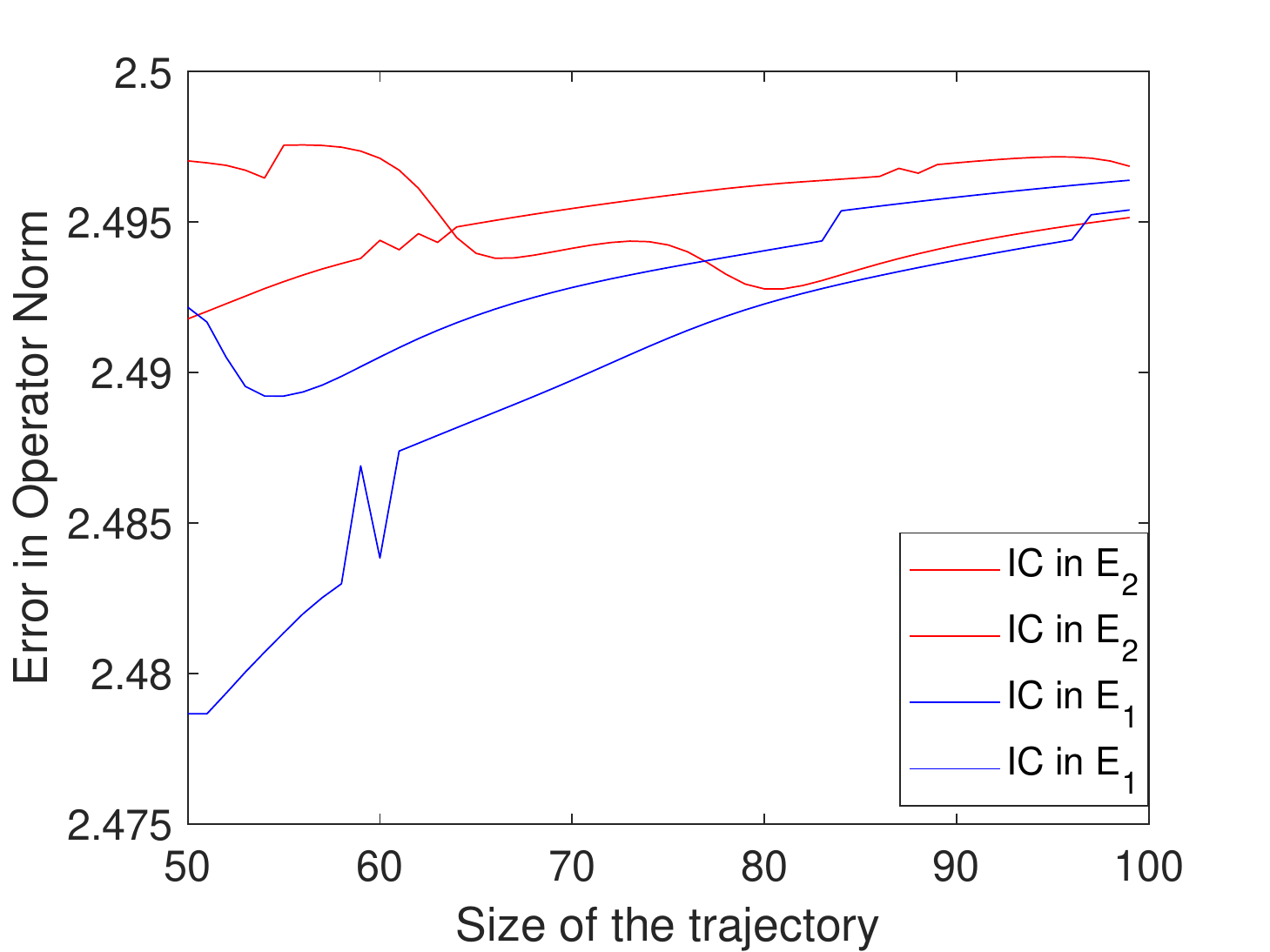}
\vspace{-18pt}
\caption{Ordinary Least Squares on regular explosive systems with  isotropic excitations}
\label{fig:ols_incoc}}
\end{figure}
$A \in \mathbb{R}^{50 \times 50}$ with only two distinct eigenvalues $\lambda_1=1.5$ with Block size $|B_{\lambda_1}|=47$ and $\lambda_2=-0.5$ with Block size of $|B_{\lambda_2}|=3$, using direct sum decomposition: 
\begin{equation}
\label{eq:case-A}
    A=A_{\lambda_1} \oplus A_{\lambda_2}.
\end{equation}
Similarly the state space can be represented as direct sum  decomposition of two-A-invariant subspaces:
\begin{equation}
    \mathbb{R}^{n}= M_{\lambda_1}\oplus M_{\lambda_2}.
\end{equation}
We can get span of two invariant subsapces related to the two Jordan blocks by computing $[M,D]=eig(A)$, where $M=[M_{\lambda_1} \hspace{3pt} M_{\lambda_2}]$ and orthogonal projections onto $M_{\lambda_1}$ and $M_{\lambda_2}$ by applying projection operator $E_1=M_{\lambda_1}(M_{\lambda_1}^{T}M_{\lambda_1})^{\dagger}M_{\lambda_1}^{T}$ and $ E_2= M_{\lambda_2}(M_{\lambda_2}^{T}M_{\lambda_2})^{\dagger}M_{\lambda_2}^{T}$, respectively and $\dagger$ is used to represent pseudo-inverse. In Fig \ref{fig:ols_incoc} we perform four different simulations and record error in operator norm as the length of trajectory increase from $50$ to $100$. Curves in blue correspond to initial condition being an orthogonal projection defined by $E_1$ on randomly sampled isotropic Gaussian and similarly curves in red correspond to initial condition being an orthogonal projection defined by $E_2$ on randomly sampled isotropic Gaussian. As suggested by Talagrands-inequality for variance of least singular value of explosive systems in Theorem (\ref{thm:OLSvarhuge}), OLS is inconsistent.
\begin{remark}
    System considered in the case study is completely regular (as defined in \cite{faradonbeh2018finite}) but still least square estimates are incorrect (Gaussian concentration of measure phenomenon is at work)
\end{remark}

\section{Conclusion and Future Work}
\label{sec:conclusion}
In this paper, we began with the study of the correlation between two distinct time realization of stable linear systems in high dimensions with excitations of isotropic Gaussian. We employ a novel approach, where rather than just basing our analysis on the magnitude of eigenvalue, we also took into consideration the \emph{geometric content} related to the operator via information on the size of its invariant sub-spaces w.r.t state transition matrix. Which not only provided us with a geometric insight but also improved concentration results via sampling from sub-trajectory with smaller gaps (compared to what was previously believed ) between two almost uncorrelated samples. Leveraging on these geometric insights along with Talagrands' inequality in the later half of the paper we analyse inconsistency issues with OLS for explosive systems with isotropic Gaussian as excitation signal. It turns out that problem of system identification for high dimensional dynamical systems is inseparable from concentration of measure phenomenon and isoperimetric inequalities in high dimensions, where we were able to show issues in OLS with isotropic excitations by employing results associated to concentration of projections of Gaussian measures on large sub-spaces. In the future we intend on formalizing these results via tools in high dimensional geometry.


\bibliography{root}
\bibliographystyle{IEEEtran}

\end{document}